%% file: main.tex
\documentclass{article}

 \usepackage[main, final]{neurips_2026}

\usepackage[utf8]{inputenc} 
\usepackage[T1]{fontenc}    
\usepackage{hyperref}       
\usepackage{url}            
\usepackage{booktabs}       
\usepackage{amsfonts}       
\usepackage{nicefrac}       
\usepackage{microtype}      
\usepackage{xcolor}  
\usepackage{multirow}
\usepackage{enumitem}
\usepackage{amsmath,amssymb,amsthm}
\usepackage{graphicx}
\usepackage{wrapfig}
\usepackage{subcaption}

\theoremstyle{plain}
\newtheorem{theorem}{Theorem}[section]
\newtheorem{proposition}[theorem]{Proposition}

\theoremstyle{definition}

\theoremstyle{remark}

\theoremstyle{problem}

\title{LatentRouter: Can We Choose the Right Multimodal Large Language Model Before Seeing Its Answer?}

\author{%
  Xueqi Cheng\\
  Department of Computer Science\\
  Florida State University\\
  \texttt{xc25@fsu.edu} \\
\And
  Yushun Dong\textsuperscript{\(\dagger\)} \\
  Department of Computer Science\\
  Florida State University\\
  \texttt{yushun.dong@fsu.edu} \\
}

\begin{document}
\maketitle
\renewcommand{\thefootnote}{\fnsymbol{footnote}}
\footnotetext[2]{Corresponding author.}

\begin{abstract}
Multimodal large language models (MLLMs) have heterogeneous strengths across OCR, chart understanding, spatial reasoning, visual question answering, cost, and latency. Effective MLLM routing therefore requires more than estimating query difficulty: a router must match the multimodal requirements of the current image-question input with the capabilities of each candidate model. We propose \textsc{LatentRouter}, a router that formulates MLLM routing as counterfactual multimodal utility prediction. Given an image-question query, \textsc{LatentRouter} extracts learned multimodal routing capsules, represents each candidate MLLM with a model capability token, and performs latent communication between these states to estimate how each model would perform if selected. A distributional outcome head predicts model-specific counterfactual quality, while a bounded capsule correction refines close decisions without allowing residual signals to dominate the prediction. The resulting utility-based policy supports performance-oriented and performance-cost routing, and handles changing candidate pools through shared per-model scoring with availability masking. Experiments on MMR-Bench and VL-RouterBench show that \textsc{LatentRouter} outperforms fixed-model, feature-level, and learned-router baselines. Additional analyses show that the gains are strongest on multimodal task groups where model choice depends on visual, layout-sensitive, or reasoning-oriented requirements, and that latent communication is the main contributor to the improvement. The code is available at: 
\url{https://github.com/LabRAI/LatentRouter}.
\end{abstract}

\input{intro}

\input{method}
\input{exp_setup}
\input{result}
\input{related}
\input{con}

\bibliographystyle{plain}
\bibliography{ref}

\appendix

\input{appendix}

\input{reproduce}

\end{document}

%% file: intro.tex
\section{Introduction}\label{sec:intro}

Multimodal large language models (MLLMs) are rapidly diversifying in architecture~\cite{yin2023mllm_survey,liang2024mllm_guide,jin2024efficient_mllm_survey}, scale~\cite{wang2024qwen2vl,chen2024internvl25,bai2025qwen25vl}, cost~\cite{jin2024efficient_mllm_survey,qiu2025efficient_mllm_serving,aneja2026phi4rv}, latency~\cite{jin2024efficient_mllm_survey,qiu2025efficient_mllm_serving,aneja2026phi4rv}, and capability profile~\cite{wang2024qwen2vl,chen2024internvl25,li2024llavaonevision}. These models now support a broad range of vision-language tasks, including visual question answering~\cite{goyal2017making,hudson2019gqa,singh2019textvqa}, optical character recognition~\cite{liu2023ocrbench,fu2025ocrbenchv2,yang2024ccocr}, document and chart understanding, spatial reasoning, and multimodal mathematical reasoning~\cite{lu2023mathvista,wang2024mathvision,zhang2024mathverse}. However, stronger average performance does not imply that one MLLM is uniformly best. Different models often specialize in different types of multimodal inputs: one model may be reliable on dense text and document images, another may be stronger on charts or diagrams, and another may offer a better quality-cost tradeoff for general visual question answering~\cite{xie2025mmeunify,ma2026mmrbench,huang2025vlrouterbench, cheng2025bts, cheng2025misleader}. This heterogeneity makes MLLM routing especially important in real-world multimodal systems where many models are available, but always calling the strongest model can be unnecessarily expensive, slow, or impractical under deployment constraints.

Despite this practical need, MLLM routing remains underdeveloped compared with text-only LLM routing~\cite{ma2026mmrbench,huang2025vlrouterbench}. Most existing routers for multi-LLM systems were designed for language-only inputs, where routing decisions are typically based on text-derived query features~\cite{ding2024hybridllm,ong2024routellm}, confidence or uncertainty estimates~\cite{zhang2025uncertaintyrouting,aggarwal2023automix}, cascading policies~\cite{chen2023frugalgpt,dekoninck2024unifiedrouting}, or hierarchical routing structures~\cite{varangotreille2025routing,tian2026haps}. These approaches are useful when query difficulty can be inferred mainly from text, but MLLM routing is different: the best model often depends on multimodal routing signals that arise from the image, the question, and their interaction. For example, the router may need to recognize whether the input requires reading dense text, interpreting a chart, following spatial relations, comparing objects, or reasoning over a diagram before it can decide which candidate MLLM is likely to succeed. A natural workaround is to convert the image into explicit textual signals, such as captions, OCR transcripts, or predefined task labels, but these signals can be costly~\cite{hao2025coconut,he2026plume}, incomplete~\cite{kamath2023text_bottleneck,he2026plume}, and insufficient when the routing-relevant information is primarily visual~\cite{shao2026modalmixed,xu2025visualplanning, xu2026synhat}. As a result, MLLM routing faces two coupled challenges: the router must preserve compact but informative multimodal routing signals from the input, and it must align these query requirements with heterogeneous model capabilities to predict which MLLM is most useful for the current query.

To address these challenges, we propose \textsc{LatentRouter}, a latent capability-matching router that formulates MLLM routing as \emph{counterfactual multimodal utility prediction}. Instead of directly mapping a query to a fixed model ID, \textsc{LatentRouter} predicts how each candidate MLLM would perform on the current image-question input and selects the model with the highest predicted utility. To preserve compact multimodal routing signals, it extracts a small set of learned multimodal routing capsules from image and question representations. To align query requirements with model capabilities, it represents each candidate MLLM with a model capability token constructed from calibration, cost, latency, and slice-level behavior, and then performs latent communication between routing capsules and model tokens. The final model states are used for distributional counterfactual outcome prediction, while a bounded capsule correction refines close decisions without allowing noisy residual signals to dominate. The resulting utility-based rule supports performance-oriented routing, performance-cost routing, and changing model pools through shared per-model scoring with availability masking. Our main contributions are as follows:
\begin{itemize}[leftmargin=*]
    \item \textbf{A counterfactual utility formulation for MLLM routing.}
    We formulate MLLM routing as predicting each candidate model's outcome on the same multimodal query. This formulation supports performance-oriented routing, performance-cost routing, and changing model pools through shared per-model utility scoring with availability masking.

    \item \textbf{\textsc{LatentRouter}: a latent capability-matching architecture.}
    We introduce a router that represents the query with learned multimodal routing capsules and represents candidate MLLMs with model capability tokens. Through latent communication, the router aligns query requirements with model capabilities and produces model-specific counterfactual utility estimates.

    \item \textbf{Comprehensive evaluation on MLLM routing benchmarks.}
    Experiments on MMR-Bench and VL-RouterBench show that \textsc{LatentRouter} outperforms fixed-model, feature-level, and learned-router baselines. Further analyses examine counterfactual outcome prediction, component ablations, task-group behavior, efficiency, and model-pool flexibility.
\end{itemize}

%% file: method.tex



\section{Methodology}
\label{sec:method}

\begin{figure}[t]
    \centering
    \includegraphics[width=\linewidth]{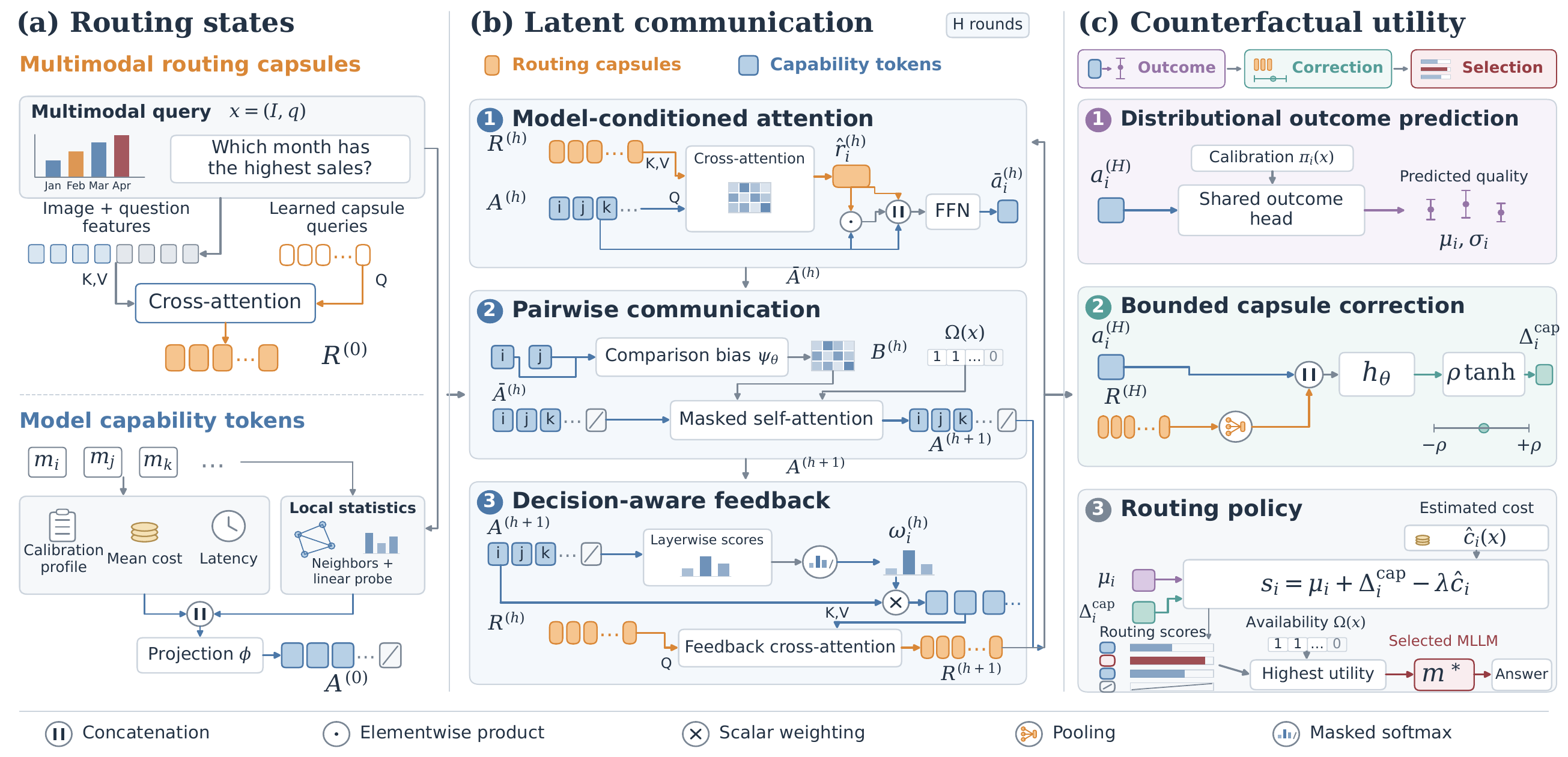}
    \caption{
    Overview of \textsc{LatentRouter}. 
(a) Multimodal routing capsules encode the image--question query, while model capability tokens combine model metadata and query-specific calibration statistics.
(b) Over $H$ rounds, model-conditioned attention, availability-masked pairwise communication, and decision-aware feedback iteratively update the two sets of states.
(c) A shared outcome head predicts each model's counterfactual answer-quality distribution. A bounded capsule correction refines the predicted mean, and the routing policy selects the available model with the highest cost-adjusted utility. Only the selected MLLM is invoked to generate an answer.
Orange and blue denote routing capsules and model capability tokens, respectively.
    }
    \label{fig:latentrouter_training_workflow}
\end{figure}

We introduce \textsc{LatentRouter}, a latent capability-matching router for selecting among a pool of candidate MLLMs. The central idea is to treat routing as counterfactual utility prediction: before selecting a model, the router estimates how each available candidate would perform on the same image-question query. \textsc{LatentRouter} represents the query with learned multimodal routing capsules, represents each candidate MLLM with a model capability token, and aligns these two sets of latent states through latent communication. The overview is shown in Figure~\ref{fig:latentrouter_training_workflow}.

\subsection{Problem Formulation}
\label{sec:problem_formulation}

Let $x=(I,q)$ denotes a multimodal query, where $I$ is an image and $q$ is a natural-language question. Let $\mathcal{M}=\{m_1,\ldots,m_K\}$ be the full candidate MLLM pool. For each query $x$, only a subset of models may be available; we denote the available index set by $\Omega(x)\subseteq\{1,\ldots,K\}$. A routing trace contains the query $x$, the availability mask $\Omega(x)$, the observed answer quality $y_i(x)$ and answer cost $c_i(x)$ for each available model $m_i$, and model-level deployment attributes such as latency. We normalize answer quality and cost to make utilities comparable across models and benchmarks.

Given a cost tradeoff parameter $\lambda\geq 0$, the utility of model $m_i$ on query $x$ is $u_i(x)=y_i(x)-\lambda c_i(x)$, where $c_i(x)$ is the normalized cost of answering $x$ with $m_i$. When $\lambda=0$, the objective is performance-oriented routing. When $\lambda>0$, the objective is performance-cost routing.

A conventional router can be written as a policy $\pi_\theta(x,\mathcal{M})\in\mathcal{M}$, which directly outputs a selected model. However, this notation hides the counterfactual structure of the problem. At routing time, the router chooses a model before observing its answer, so it must estimate how each model would perform on the same input. We therefore learn a scoring function $s_\theta(x,m_i,\mathcal{M})\approx u_i(x)$, and select $\pi_\theta(x,\mathcal{M})
=
\arg\max_{i\in\Omega(x)} s_\theta(x,m_i,\mathcal{M})$. This formulation uses the full outcome vector over the candidate pool during training. It also allows the same router to score variable model pools by masking unavailable models.

\subsection{\textsc{LatentRouter}}
\label{sec:main_method}

\textsc{LatentRouter} implements $s_\theta(x,m_i,\mathcal{M})$ through latent capability matching. It consists of five components: multimodal routing capsules, model capability tokens, latent communication layers, a distributional outcome predictor, and a bounded capsule correction.

\paragraph{Why query-model interaction is necessary.}
The best MLLM depends on how the current query requirements match each candidate model's capabilities. A router that assigns independent query and model scores cannot express query-dependent model preference reversals: the same model pair would have the same ordering for every input. However, MLLM routing often requires such reversals, e.g., one model may be better for OCR-heavy inputs while another may be better for chart or spatial reasoning. We therefore use an interaction-dependent score $s_\theta(x,m_i,\mathcal{M}) = f_\theta(R(x),a_i,A)$, where $R(x)$ denotes multimodal routing capsules, $a_i$ is the capability token of model $m_i$, and $A$ is the set of model capability tokens. Appendix~\ref{app:proofs_method} formalizes why additive scoring is insufficient.

\paragraph{Multimodal routing capsules.}
MLLM routing requires identifying which aspects of the image-question input determine model choice, such as visual text, layout, object relations, spatial cues, or image-question alignment. A single pooled query vector may blur these factors, while different candidate models may depend on different ones. \textsc{LatentRouter} therefore represents the query as a small set of learned multimodal routing capsules.

Let $V=\{v_1,\ldots,v_N\}$ be image tokens and $Q=\{q_1,\ldots,q_L\}$ be question tokens from a multimodal encoder. Given $C$ learned capsule queries $\{z_k^{\mathrm{cap}}\}_{k=1}^{C}$, we extract capsules by cross-attending each query to the joint image-question sequence:
\[
r_k^{(0)}
=
\mathrm{CrossAttn}(z_k^{\mathrm{cap}}, [V;Q]),
\qquad
k=1,\ldots,C.
\]
The initial capsule set is $R^{(0)}(x)=\{r_1^{(0)},\ldots,r_C^{(0)}\}$. These capsules are learned latent routing states, not predefined task labels. They are trained end-to-end through the per-model outcome prediction and routing losses, so they retain information only insofar as it helps predict which candidate MLLM will perform well on the current query.

\paragraph{Model capability tokens.}
For each candidate model $m_i$, \textsc{LatentRouter} constructs a model capability token from model-level metadata and calibration statistics as $a_i^{(0)}=\phi([p_i;\bar c_i;\ell_i;b_i(x)])$. Here, $p_i$ is the calibration profile of model $m_i$, $\bar c_i$ is the average normalized cost of $m_i$ on the training split, $\ell_i$ is latency, and $b_i(x)$ contains slice-level or pairwise calibration statistics. Besides benchmark-annotated slices, we treat the $k$ calibration queries most similar to $x$ as a query-specific slice: $b_i(x)$ includes the accuracy of $m_i$ on this neighborhood and its deviation from the model's overall accuracy, together with a linear-probe estimate of $m_i$'s accuracy on $x$, summarized into a calibrated estimate $\pi_i(x)$ (Appendix~\ref{app:model_capability_profiles}). The projection $\phi(\cdot)$ maps these features into the same hidden dimension as the routing capsules. The model-token set is $A^{(0)}=\{a_1^{(0)},\ldots,a_K^{(0)}\}$. Unavailable models are included in notation but masked in attention and scoring. This design avoids a fixed $K$-way model classifier: the same scoring head is shared across model tokens, and routing is performed over the currently available pool.

\paragraph{Latent communication.}
The routing capsules encode query-side multimodal requirements, while model tokens encode candidate MLLM capabilities. \textsc{LatentRouter} aligns these states through \emph{counterfactual latent communication}: each model token reads the query capsules from its own capability perspective, model tokens compare with other available candidates, and routing capsules receive feedback from the most decision-relevant model states. This makes the communication routing-specific while keeping the intermediate process in continuous latent space.

Given routing capsules $R^{(h)}=\{r_1^{(h)},\ldots,r_C^{(h)}\}$ and model tokens $A^{(h)}=\{a_1^{(h)},\ldots,a_K^{(h)}\}$, each communication layer performs three updates. First, each model token reads the capsules through model-conditioned attention:
\[
\alpha_{ik}^{(h)}
=
\mathrm{softmax}_{k}
\left(
\frac{(W_q a_i^{(h)})^\top (W_k r_k^{(h)})}{\sqrt{d}}
\right),
\qquad
\hat r_i^{(h)}
=
\sum_{k=1}^{C}\alpha_{ik}^{(h)} W_v r_k^{(h)} .
\]
The model token is then updated as $\bar a_i^{(h)}
=
\mathrm{FFN}
\left(
[a_i^{(h)};\hat r_i^{(h)};a_i^{(h)}\odot \hat r_i^{(h)}]
\right)$. This produces a candidate-specific view of the query: different MLLMs can attend to different routing signals, such as visual text, layout, chart structure, or spatial cues.

Second, model tokens compare with one another through availability-masked pairwise communication. For each pair, we compute a comparison bias
\[
b_{ij}^{(h)}
=
\psi_\theta
\left(
[\bar a_i^{(h)};\bar a_j^{(h)};
\bar a_i^{(h)}-\bar a_j^{(h)};
\bar a_i^{(h)}\odot \bar a_j^{(h)}]
\right),
\]
and inject it into masked self-attention $A^{(h+1)}
=
\mathrm{SelfAttn}_{\Omega(x)}
\left(
\{\bar a_1^{(h)},\ldots,\bar a_K^{(h)}\};
B^{(h)}
\right)$. The mask removes unavailable models, while the pairwise bias makes the update explicitly comparative. Setting $B^{(h)}=0$ recovers standard masked self-attention.

Third, routing capsules receive decision-aware feedback from the updated model tokens. A lightweight layerwise score $s_i^{(h)}$ estimates the current utility of each available model, and feedback weights are assigned as $\omega_i^{(h)}
=
\mathrm{softmax}_{i\in\Omega(x)}
\left(
-\frac{\max_j s_j^{(h)}-s_i^{(h)}}{\tau}
\right)$. The routing capsules are updated by attending to the weighted model states $R^{(h+1)}
=
\mathrm{CrossAttn}
\left(
R^{(h)},
\{\omega_i^{(h)}a_i^{(h+1)}\}_{i\in\Omega(x)}
\right)$. This feedback emphasizes candidates with high current utility estimates, making the query representation aware of which multimodal signals distinguish the available models. After $H$ layers, each final model token $a_i^{(H)}$ is a query-conditioned, pool-aware representation used for the next step.

\paragraph{Distributional counterfactual outcome prediction.}
For each available model, the final model token predicts a distribution over the model's answer quality, parameterized relative to the calibration estimate $\pi_i(x)$ carried by the capability token:
\[
\mu_i=\psi^{-1}\big(\psi(\pi_i(x))+g_\theta(a_i^{(H)})\big),
\]
where the outcome family and its link $\psi$ are chosen according to the type of answer quality in the target task, and $\sigma_i$ is the predictive standard deviation of that family. This distributional prediction is useful because routing traces provide more than the identity of the best model: they provide a full vector of observed outcomes over candidate MLLMs. Learning this vector encourages the router to model the relative strengths and weaknesses of the pool.

\paragraph{Bounded capsule correction.}
The distributional head provides a stable base estimate, but some fine-grained multimodal routing signals may still be underused. We therefore add a bounded correction from the final routing capsules:
\[
\tilde{\mu}_i
=
\mu_i+\Delta_i^{\mathrm{cap}},
\qquad
\Delta_i^{\mathrm{cap}}
=
\rho\cdot
\tanh\!\left(
h_\theta([a_i^{(H)};\mathrm{Pool}(R^{(H)})])
\right).
\]
The scalar $\rho$ limits the correction magnitude. Thus, the residual can refine close model comparisons, but cannot arbitrarily override confident base predictions. Appendix~\ref{app:proofs_method} gives more detailed proof.

\paragraph{Routing policy.}
The final routing score is the corrected predicted quality minus the cost penalty:
\[
s_i(x)=\tilde{\mu}_i(x)-\lambda \hat c_i(x),
\]
where $\hat c_i(x)$ estimates the answer cost $c_i(x)$; since $c_i(x)$ depends on the not-yet-generated output, we fit $\hat c_i(x)$ by ridge regression on the costs observed in the training traces. The selected model is $\pi_\theta(x,\mathcal{M}) = \arg\max_{i\in\Omega(x)} s_i(x)$. Changing $\lambda$ yields different quality-cost operating points. All unavailable models are masked before the final maximization.

\subsection{Training and Deployment}
\label{sec:training_objective}

\textbf{Training}. \textsc{LatentRouter} is trained on model-pool outcome traces. For each query, losses are computed only over available models $i\in\Omega(x)$. The primary loss is the negative log-likelihood of the observed quality under the predicted distribution. This trains the distributional outcome predictor to estimate the mean outcome and its uncertainty. To align prediction with routing decisions, we add a pairwise ranking loss over model utilities:
\[
\mathcal{L}_{\mathrm{pair}}
=
\sum_{\substack{i,j\in\Omega(x)\\ u_i>u_j}}
\log\left(1+\exp(-(s_i-s_j))\right).
\]
This encourages models with higher observed utility to receive higher predicted routing scores. We then use a listwise loss to match the predicted utility distribution to the observed utility distribution:
\[
\mathcal{L}_{\mathrm{list}}
=
\mathrm{CE}
\left(
\mathrm{softmax}(\mathbf{u}_{\Omega}/\tau),
\mathrm{softmax}(\mathbf{s}_{\Omega}/\tau)
\right),
\]
where $\mathbf{u}_{\Omega}$ and $\mathbf{s}_{\Omega}$ are utility and score vectors restricted to available models, and $\tau$ is a temperature. Finally, we include direct utility regression and residual regularization:
\[
\mathcal{L}_{\mathrm{util}}
=
\sum_{i\in\Omega(x)}(s_i-u_i)^2,
\qquad
\mathcal{L}_{\mathrm{res}}
=
\sum_{i\in\Omega(x)}(\Delta_i^{\mathrm{cap}})^2.
\]
The full training objective is $\mathcal{L}
=
\mathcal{L}_{\mathrm{dist}}
+
\alpha\mathcal{L}_{\mathrm{pair}}
+
\beta\mathcal{L}_{\mathrm{list}}
+
\gamma\mathcal{L}_{\mathrm{util}}
+
\eta\mathcal{L}_{\mathrm{res}}$. The hyperparameters $\alpha,\beta,\gamma,\eta,\rho$, and $\tau$ are selected on the validation split and fixed during testing. In the ranking and regression losses, $s_i$ and $u_i$ are computed with $\lambda=0$, i.e., $s_i=\tilde{\mu}_i(x)$ and $u_i=y_i(x)$; the cost tradeoff enters only through $\hat c_i(x)$ at routing time, so a single trained router serves every $\lambda$.

\textbf{Deployment.} At deployment time, \textsc{LatentRouter} receives a new image-question query and an available candidate model pool. It constructs routing capsules from the query, constructs model capability tokens for the available MLLMs, applies latent communication, predicts corrected outcomes $\tilde{\mu}_i$, and selects the model with the highest score $s_i(x)=\tilde{\mu}_i(x)-\lambda\hat c_i(x)$. Because the outcome predictor is shared across model tokens and unavailable models are masked, the same architecture can be applied to different model subsets without changing the classifier dimension.

%% file: exp_setup.tex
\section{Experimental Setup}
\label{sec:exp_setup}

\subsection{Datasets, Models, and Metrics}
\label{sec:datasets_models_metrics}

We evaluate \textsc{LatentRouter} on two most recent standard MLLM routing benchmarks: MMR-Bench~\cite{ma2026mmrbench} and VL-RouterBench~\cite{huang2025vlrouterbench}. Both benchmarks provide image-question queries, candidate MLLM pools, per-model answer-quality annotations, and model-cost information. MMR-Bench covers OCR, general VQA, and multimodal reasoning tasks, while VL-RouterBench covers general VQA, STEM reasoning, chart/OCR, and document understanding. Following the original benchmark settings, we evaluate 11 candidate models on MMR-Bench and 17 candidate models on VL-RouterBench; the full model lists are provided in Appendix~\ref{app:exp_details}.

We report results under two settings. The \emph{performance-oriented} setting evaluates whether a router selects high-quality models, while the \emph{performance-cost} setting evaluates the quality-cost tradeoff induced by cost-aware routing. We follow the official benchmark metrics: MMR-Bench reports nAUC over the cost-quality frontier, and VL-RouterBench reports its official Rank Score.

\subsection{Baselines}
\label{sec:baselines}

We compare \textsc{LatentRouter} with three groups of baselines:

\begin{itemize}[leftmargin=*]
    \item \textbf{Training-free baselines.}
    We include \textbf{Oracle}, \textbf{Strongest}, \textbf{Cheapest}, and \textbf{Random}. Oracle uses ground-truth outcomes and serves only as an upper bound; Strongest selects the best validation model; Cheapest selects the lowest-cost model; and Random samples uniformly from the pool.

    \item \textbf{Feature-level routers.}
    We compare with standard feature-based routers, including \textbf{KNNRouter}, \textbf{MLPMFRouter}, and \textbf{KMeansRouter}. These methods are trained on the same benchmark-provided router features and splits, and test whether conventional feature-level routing is sufficient without latent query-model interaction.

    \item \textbf{Learned router baselines.}
    We compare with the most recent MLLM router \textbf{ECVL}~\cite{tang2025ecvl}, as well as four representative text-based LLM routers, \textbf{RouteLLM}~\cite{ong2024routellm} \textbf{RouterDC}~\cite{chen2024routerdc}, \textbf{GraphRouter}~\cite{feng2024graphrouter} and \textbf{R2-Router}~\cite{xue2026r2}. All learned baselines are adapted to the same candidate pools, splits, cost annotations, and routing features, and are implemented according to their original papers.
\end{itemize}

\subsection{Implementation Details}
\label{sec:implementation_details}

To ensure a fair comparison, \textsc{LatentRouter} and all trainable baselines use the same benchmark-provided image features, question features, side features, model-pool metadata, cost annotations, and model-outcome traces. \textsc{LatentRouter} projects query features into a shared hidden space and summarizes them with $C=7$ learned multimodal routing capsules. Each candidate MLLM is represented by a model capability token built from calibration profile, normalized cost, latency, and available slice-level or pairwise statistics, all computed only from the training/calibration split to avoid test-set leakage. We use $H=2$ latent communication layers, where model tokens attend to routing capsules, interact through availability-masked self-attention, and update the capsules through feedback attention. A shared distributional outcome head predicts per-model mean and uncertainty, followed by bounded capsule correction. All trainable methods are trained on the official training split, tuned on validation, and evaluated on test over three random seeds.

%% file: result.tex
\begin{table}[t]
\centering
\setlength\tabcolsep{9pt}
\caption{Main routing performance on MMR-Bench and VL-RouterBench. MMR-Bench reports nAUC ($\uparrow$), and VL-RouterBench reports Rank Score ($\uparrow$). Best non-oracle results are in \textbf{bold}; second-best non-oracle results are \underline{underlined}.}
\label{tab:main_results}
\begin{tabular}{lcccc}
\toprule
\multirow{2}{*}{Router} & \multicolumn{2}{c}{MMR-Bench} & \multicolumn{2}{c}{VL-RouterBench} \\
\cmidrule(lr){2-3} \cmidrule(lr){4-5}
& Oriented & Cost & Oriented & Cost \\
\midrule
Oracle & 0.918 $\pm$ 0.000 & 0.899 $\pm$ 0.000 & 0.932 $\pm$ 0.000 & 0.932 $\pm$ 0.000 \\
Strongest & 0.758 $\pm$ 0.000 & 0.758 $\pm$ 0.000 & 0.776 $\pm$ 0.000 & 0.776 $\pm$ 0.000 \\
Cheapest & 0.438 $\pm$ 0.000 & 0.438 $\pm$ 0.000 & 0.675 $\pm$ 0.000 & 0.675 $\pm$ 0.000 \\
Random & 0.608 $\pm$ 0.005 & 0.606 $\pm$ 0.004 & 0.607 $\pm$ 0.004 & 0.608 $\pm$ 0.004 \\
\midrule
KNNRouter & 0.726 $\pm$ 0.000 & 0.724 $\pm$ 0.000 & 0.770 $\pm$ 0.000 & 0.768 $\pm$ 0.000 \\
MLPMFRouter & 0.739 $\pm$ 0.002 & 0.738 $\pm$ 0.005 & 0.762 $\pm$ 0.002 & 0.760 $\pm$ 0.003 \\
KMeansRouter & 0.758 $\pm$ 0.000 & \underline{0.759 $\pm$ 0.001} & 0.781 $\pm$ 0.004 & 0.782 $\pm$ 0.004 \\
\midrule
ECVL & 0.748 $\pm$ 0.000 & 0.748 $\pm$ 0.000 & 0.779 $\pm$ 0.000 & 0.762 $\pm$ 0.000 \\
RouteLLM & \underline{0.759 $\pm$ 0.008} & 0.634 $\pm$ 0.002 & 0.780 $\pm$ 0.004 & 0.688 $\pm$ 0.001 \\
RouterDC & 0.582 $\pm$ 0.003 & 0.619 $\pm$ 0.011 & 0.732 $\pm$ 0.002 & 0.683 $\pm$ 0.002 \\
GraphRouter & 0.734 $\pm$ 0.001 & 0.714 $\pm$ 0.009 & 0.783 $\pm$ 0.011 & 0.773 $\pm$ 0.002 \\
R2-Router & 0.736 $\pm$ 0.006 & 0.725 $\pm$ 0.010 & \underline{0.790 $\pm$ 0.002} & \underline{0.785 $\pm$ 0.004} \\
\midrule
\textbf{\textsc{LatentRouter}} & \textbf{0.792 $\pm$ 0.001} & \textbf{0.786 $\pm$ 0.001} & \textbf{0.815 $\pm$ 0.001} & \textbf{0.812 $\pm$ 0.003} \\
\bottomrule
\end{tabular}
\end{table}

\section{Experiment Results}
\label{sec:result}

We evaluate \textsc{LatentRouter} around four questions.
\textbf{RQ1}: Does \textsc{LatentRouter} improve routing performance across benchmarks and cost settings?
\textbf{RQ2}: How accurately does LatentRouter predict per-model outcomes?
\textbf{RQ3}: Which components are responsible for the improvement?
\textbf{RQ4}: Does the router support changing candidate model pools while remaining efficient?

\subsection{Main Routing Performance}
\label{sec:main_results}

For \textbf{RQ1}, Table~\ref{tab:main_results} reports aggregate results on MMR-Bench and VL-RouterBench. \textsc{LatentRouter} achieves the best non-oracle result in all four settings. Compared with the strongest non-oracle baseline in each column of the Table~\ref{tab:main_results}, the improvement is statistically significant under an approximate two-sided Welch test computed from the reported mean and standard deviation over three seeds: MMR-Oriented ($p=0.0178<0.05$), MMR-Cost ($p=4.99{\times}10^{-6}<0.001$), VL-Oriented ($p=3.40{\times}10^{-4}<0.001$), and VL-Cost ($p=0.0010<0.01$). This pattern is important because the best competing baseline changes across benchmarks and settings: fixed-model selection is competitive in some cases, feature-level routers are strong on VL-RouterBench, and recent interaction-style routers such as GraphRouter and R2-Router narrow the gap. \textsc{LatentRouter} remains strongest across all of them, suggesting that the gain comes from a more effective query-model matching mechanism rather than from a benchmark-specific shortcut.

Figure~\ref{fig:cost_task_analysis} further shows that the improvement is not tied to a single operating point. Figures~\ref{fig:cost_quality_mmr} and~\ref{fig:cost_quality_vl} show that \textsc{LatentRouter} gives the strongest cost-quality frontier on both benchmarks, indicating that the learned utility estimate remains useful under different cost budgets. Figure~\ref{fig:task_group_radar} shows that the gains are most visible on OCR, Chart/Diagram, and Math tasks, where the best model depends on fine-grained multimodal signals such as visual text, layout, chart structure, and visual reasoning. The smaller gain on General VQA is expected because broadly strong MLLMs already perform well on many general examples, but the positive improvement still shows exploitable routing heterogeneity.

\begin{figure}[t]
    \centering

    \includegraphics[width=0.95\linewidth]{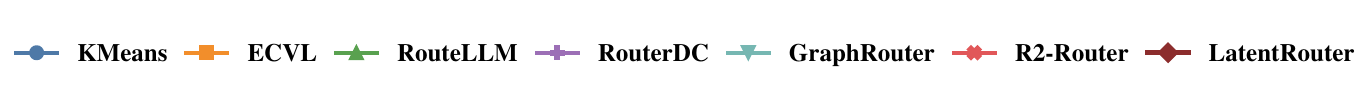}
    \vspace{-1em}

    \begin{subfigure}[t]{0.295\linewidth}
        \centering
        \includegraphics[width=\linewidth]{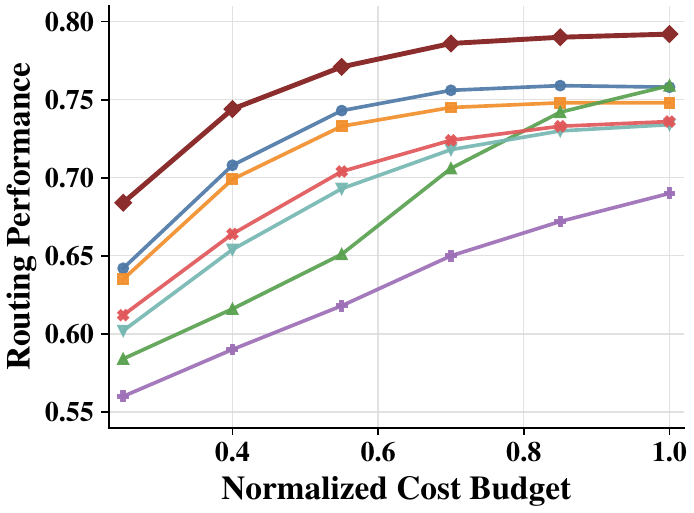}
        \caption{MMR-Bench}
        \label{fig:cost_quality_mmr}
    \end{subfigure}
    \hfill
    \begin{subfigure}[t]{0.295\linewidth}
        \centering
        \includegraphics[width=\linewidth]{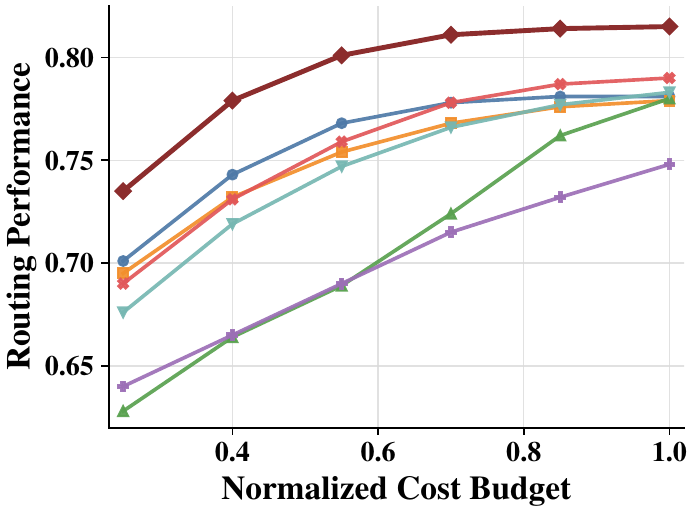}
        \caption{VL-RouterBench}
        \label{fig:cost_quality_vl}
    \end{subfigure}
    \hfill
    \begin{subfigure}[t]{0.35\linewidth}
        \centering
        \includegraphics[width=\linewidth]{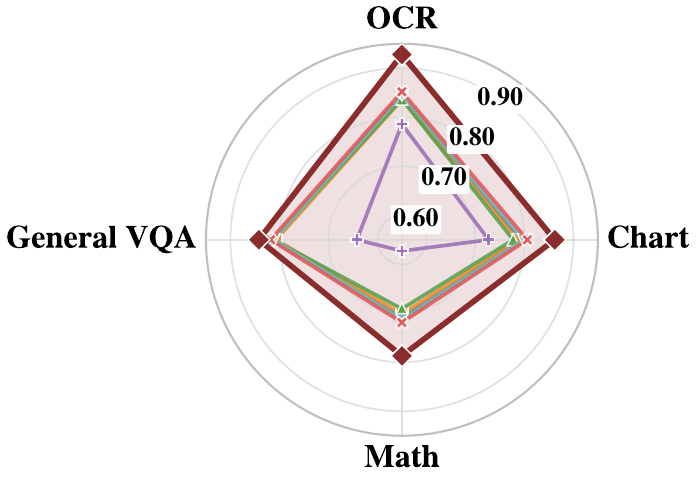}
        \caption{Task groups}
        \label{fig:task_group_radar}
    \end{subfigure}

    \caption{
    Cost-quality and task-level analysis. 
    Panels (a)--(b) show cost-quality frontiers under different cost budgets; \textsc{LatentRouter} improves the frontier on both benchmarks.
    Panel (c) groups datasets from both benchmarks by task type; gains are largest on tasks where model choice depends strongly on visual, layout-sensitive, or reasoning-oriented query requirements.
    }
    \label{fig:cost_task_analysis}
    \vspace{-1em}
\end{figure}

\subsection{Counterfactual Outcome Prediction}
\label{sec:outcome_prediction}

For \textbf{RQ2}, Table~\ref{tab:outcome_prediction} directly evaluates whether the router predicts the full counterfactual outcome vector. \textsc{LatentRouter} performs best across calibration-oriented and ranking-oriented metrics. The improvement in MSE is modest, which suggests that several methods can fit average outcome values. The more important gains appear in NDCG, Spearman correlation, and Top-3 recall, because routing depends on placing strong candidate models near the top rather than perfectly calibrating every score.

The no-communication variant isolates the role of latent communication. It uses the same query and model information, but cannot iteratively align routing capsules with model capability tokens. Its MSE remains close to the full model, while its ranking metrics are weaker. This indicates that latent communication mainly improves the relative ordering of candidate MLLMs, especially among top candidates where the final routing decision is made.

\begin{table}[t]
\centering
\small
\setlength\tabcolsep{9pt}
\caption{Counterfactual outcome prediction quality. We evaluate whether each router can predict the full per-model outcome vector.
Best results are in \textbf{bold}; second-best results are \underline{underlined}.}
\label{tab:outcome_prediction}
\begin{tabular}{lcccc}
\toprule
Method & MSE $\downarrow$ & NDCG $\uparrow$ & Spearman $\uparrow$ & Top-3 Recall $\uparrow$ \\
\midrule
Direct classifier & 0.207 $\pm$ 0.017 & 0.893 $\pm$ 0.009 & 0.347 $\pm$ 0.077 & 0.269 $\pm$ 0.121 \\
Point-score router & \underline{0.161 $\pm$ 0.012} & \underline{0.913 $\pm$ 0.011} & \underline{0.413 $\pm$ 0.078} & 0.268 $\pm$ 0.025 \\
No latent communication & 0.162 $\pm$ 0.012 & 0.912 $\pm$ 0.010 & 0.411 $\pm$ 0.079 & \underline{0.274 $\pm$ 0.033} \\
\textsc{LatentRouter} & \textbf{0.159 $\pm$ 0.012} & \textbf{0.915 $\pm$ 0.009} & \textbf{0.425 $\pm$ 0.074} & \textbf{0.283 $\pm$ 0.050} \\
\bottomrule
\end{tabular}
\end{table}

\subsection{Mechanism Analysis}
\label{sec:mechanism_analysis}

For \textbf{RQ3}, Figure~\ref{fig:mechanism_analysis} explains how the main components contribute to routing performance. Figure~\ref{fig:mechanism_analysis}a shows that moving from $H=0$ to $H=1$ gives a clear gain, meaning that even one round of communication between routing capsules and model capability tokens is useful. Performance reaches its best or near-best value at $H=2$ and then saturates. This is expected because each communication layer already contains a full query-model-pool update: model tokens read routing capsules, compare with other available model tokens, and send feedback to the capsules. After two rounds, the main routing-relevant alignment has largely been formed, so additional layers mostly repeat similar comparisons rather than introducing new information. 

Figure~\ref{fig:mechanism_analysis}b gives the component-level view. Removing latent communication causes the largest drop, confirming it as the central mechanism. Removing model capability tokens or routing capsules also hurts, showing that the router needs both model-side capability information and query-side multimodal routing signals. The distributional prediction head and bounded correction provide smaller but consistent gains: the former improves counterfactual outcome estimation, while the latter helps refine close model comparisons without destabilizing the base prediction. Together, these results show that \textsc{LatentRouter}'s gains come from the intended capsule-token matching process.

Appendix~\ref{app:additional_architecture_controls} compares
\textsc{LatentRouter} with an interaction MLP and parameter-matched
Transformer controls. These comparisons further examine the value
of structured latent communication.
Appendix~\ref{app:side_feature_control} evaluates the influence of
metadata-bearing side features.

\begin{figure}[t]
    \centering
    \begin{subfigure}{0.43\linewidth}
        \centering
        \includegraphics[width=\linewidth]{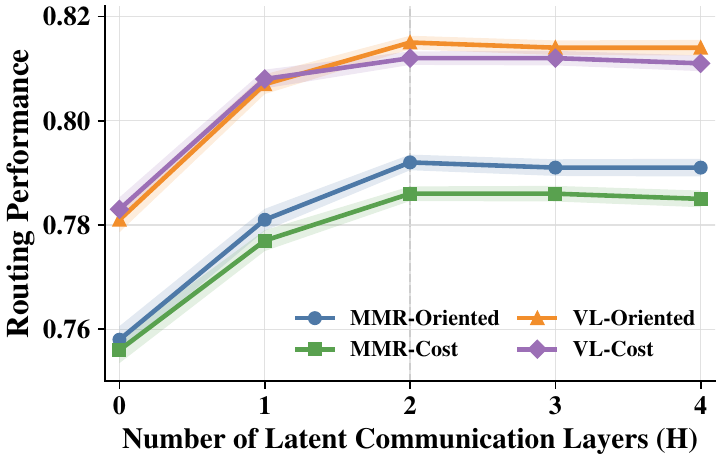}
        \caption{Communication depth}
    \end{subfigure}
    \hfill
    \begin{subfigure}{0.56\linewidth}
        \centering
        \includegraphics[width=\linewidth]{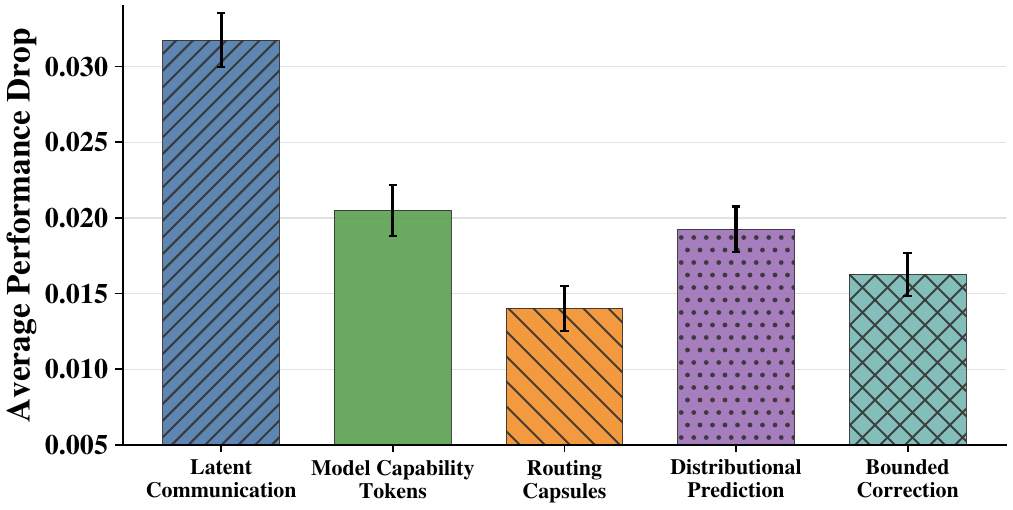}
        \caption{Ablation Study}
    \end{subfigure}
    \caption{Mechanism analysis.
(a) Performance improves quickly and saturates after $H=2$, supporting a shallow latent communication design.
(b) Ablations show latent communication is the key component, with other modules providing complementary gains.
    }
    \label{fig:mechanism_analysis}
\end{figure}

\subsection{Model-Pool Flexibility and Efficiency}
\label{sec:deployment_results}

For \textbf{RQ4}, Figure~\ref{fig:deployment_analysis} evaluates whether \textsc{LatentRouter} remains useful when the available model pool changes and whether the router itself is lightweight. In Figure~\ref{fig:deployment_analysis}a, performance drops most when the strongest model is removed, which reflects the lower ceiling of the remaining pool. Removing the cheapest model has little effect on the primary metric, while random removal and leave-one-model-out produce intermediate changes. This pattern suggests that \textsc{LatentRouter} can mask unavailable models and re-rank the remaining candidates through shared model-token scoring.

Figure~\ref{fig:deployment_analysis}b shows the same behavior from the oracle-regret perspective. Regret increases when useful models are removed, but remains bounded under non-degenerate pool changes, indicating that the router still makes meaningful selections over the available subset. Figures~\ref{fig:deployment_analysis}c and~\ref{fig:deployment_analysis}d summarize router-side efficiency. \textsc{LatentRouter} achieves the best average primary metric while maintaining the lowest router latency among learned methods. This shows that the capsule-token communication module improves model selection while keeping router overhead small relative to MLLM inference.

\begin{figure}[t]
    \centering
    \begin{subfigure}[t]{0.245\linewidth}
        \centering
        \includegraphics[width=\linewidth]{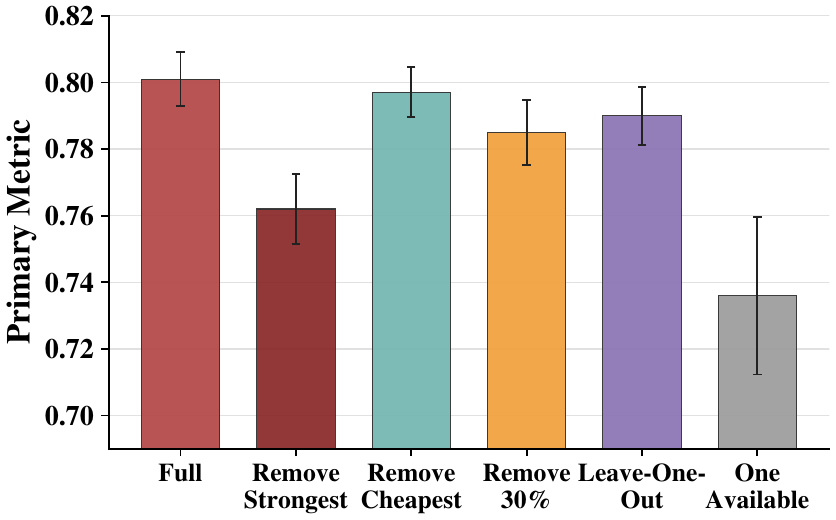}
        \caption{Changed pools}
    \end{subfigure}
    \hfill
    \begin{subfigure}[t]{0.245\linewidth}
        \centering
        \includegraphics[width=\linewidth]{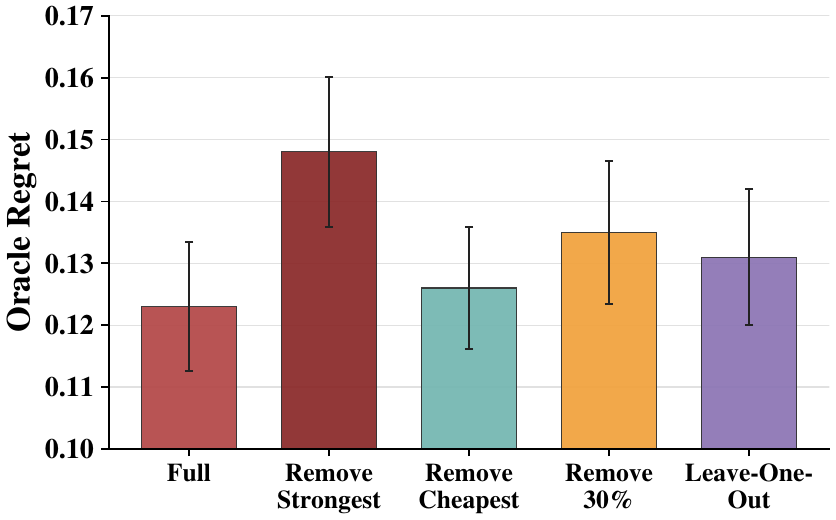}
        \caption{Oracle regret}
    \end{subfigure}
    \hfill
    \begin{subfigure}[t]{0.245\linewidth}
        \centering
        \includegraphics[width=\linewidth]{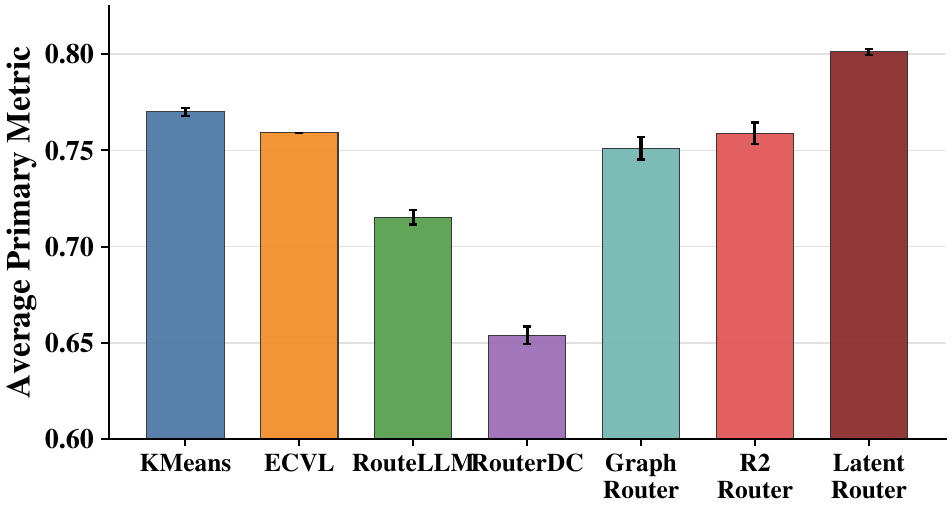}
        \caption{Router accuracy}
    \end{subfigure}
    \hfill
    \begin{subfigure}[t]{0.245\linewidth}
        \centering
        \includegraphics[width=\linewidth]{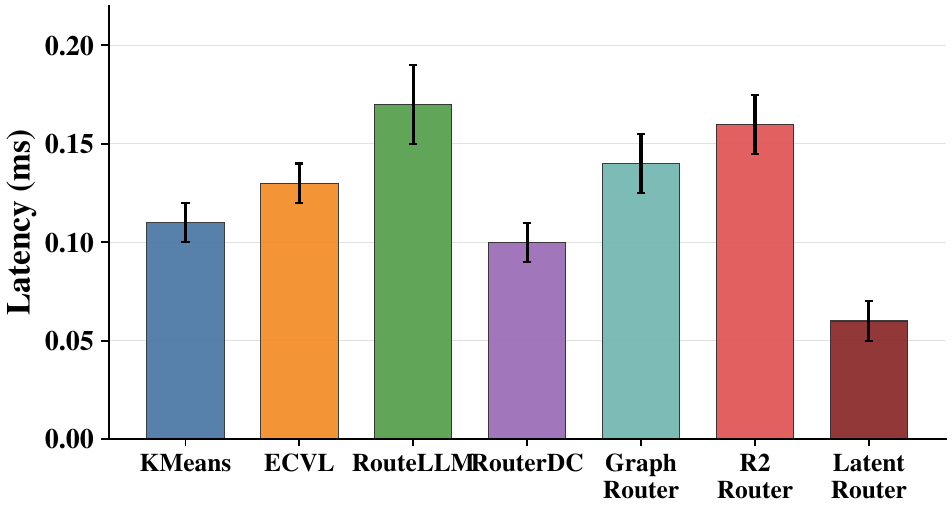}
        \caption{Router latency}
    \end{subfigure}
    \caption{
    Model-pool flexibility and efficiency analysis.
(a)--(b) \textsc{LatentRouter} handles changed candidate pools through availability masking and re-ranking.
(c)--(d) It achieves the best average primary metric with the lowest router latency among learned routers.
    }
    \label{fig:deployment_analysis}
\end{figure}

%% file: related.tex
\section{Related Work}
\label{sec:relate}

\textbf{Routing in Multi-LLM and MLLM Systems}. Routing in text-only multi-LLM systems has been widely studied as a way to improve the quality-cost tradeoff by selecting an appropriate model for each query. Prior work uses query-level preference or difficulty signals~\cite{ong2024routellm,ding2024hybridllm}, confidence or uncertainty estimates~\cite{zhang2025uncertaintyrouting}, cascading policies~\cite{chen2023frugalgpt,dekoninck2024unifiedrouting}, benchmark-driven formulations~\cite{hu2024routerbench}, or hierarchical structures~\cite{varangotreille2025routing}. These methods provide useful foundations, but they are primarily designed for language-only inputs. MLLM routing introduces additional multimodal signals, such as OCR density, chart layout, visual grounding, spatial relations, and diagram-based reasoning, that may not be recoverable from the question text alone. Recent benchmarks formalize this setting: MMR-Bench studies MLLM routing under controlled model pools, task mixtures, and cost budgets~\cite{ma2026mmrbench}, while VL-RouterBench shows substantial routability in vision-language model pools and a remaining gap to oracle selection~\cite{huang2025vlrouterbench}. Our work complements these by proposing a novel MLLM router architecture for counterfactual utility prediction. 

\textbf{Latent reasoning}. Latent reasoning studies whether intermediate computation can be carried out in continuous hidden states instead of being fully verbalized in natural language~\cite{chen2025latentcot_survey,zhu2025latent_reasoning_survey}. Coconut shows that language models can reuse hidden states as continuous reasoning states in later computation~\cite{hao2025coconut}. In multimodal reasoning, PLUME argues that verbalized chain-of-thought can create a bottleneck for rich visual information and explores latent rollouts for multimodal representations~\cite{he2026plume}; Modal-Mixed CoT studies reasoning with both text and latent visual sketches when intermediate states are not naturally expressed as text~\cite{shao2026modalmixed}; and Monet investigates latent visual reasoning through intermediate visual thoughts~\cite{wang2025monet}. More broadly, prior work on vision-language models suggests that textual or pooled representations can lose compositional visual information~\cite{kamath2023text_bottleneck}. Our work is related in motivation, but uses latent communication for a different purpose: the latent states are optimized to predict outcomes for routing, not to generate intermediate reasoning traces within a single MLLM.

%% file: con.tex
\section{Conclusion}

We introduced \textsc{LatentRouter}, an MLLM router that formulates model selection as counterfactual multimodal utility prediction. Instead of directly classifying a query into a model ID, \textsc{LatentRouter} matches multimodal routing capsules with model capability tokens through latent communication, predicts per-model counterfactual outcomes, and selects by utility. Experiments on MMR-Bench and VL-RouterBench show consistent gains over fixed-model, feature-level, and learned-router baselines in both performance-oriented and performance-cost settings. Further analyses indicate that the gains are strongest on visually and reasoning-intensive task groups, and that latent communication is the key component. Future work includes live deployment studies, adaptive calibration for newly added MLLMs, and extensions to multi-step multimodal routing.

\section*{Acknowledgment}

This material is, in part, supported by the National Science Foundation (NSF) under awards OAC-2530786 and GEO-2536578. We would also like to thank the anonymous reviewers for their constructive feedback.

%% file: appendix.tex
\section*{Technical Appendices and Supplementary Material}

\section{Additional Experimental Details}
\label{app:exp_details}

This appendix provides experimental details that are omitted from the main text for space. We first list the candidate MLLM pools used in each benchmark. We then describe how baselines are adapted to the benchmark setting, how model capability profiles are constructed, and how hyperparameters and random seeds are handled.

\subsection{Candidate Model Pools}
\label{app:model_pools}

Both benchmarks define a fixed candidate model pool and provide per-model outcome annotations for each query. We follow the original benchmark model pools without adding or removing candidate models in the main experiments. MMR-Bench contains 11 candidate MLLMs, while VL-RouterBench contains 17 candidate MLLMs. Table~\ref{tab:model_pools_appendix} lists the full model pools used in our experiments.

\begin{table}[h]
\centering
\small
\caption{Candidate MLLM pools used in each benchmark.}
\label{tab:model_pools_appendix}
\scriptsize
\begin{tabular}{clp{0.72\linewidth}}
\toprule
Benchmark & \# Models & Candidate MLLMs \\
\midrule
MMR-Bench & 11 &
Claude3-7V\_Sonnet; GeminiFlash2-5; GeminiPro2-5; Gemma3-4B; InternVL3-78B; Qwen2.5-VL-3B-Instruct; Qwen2.5-VL-7B-Instruct; Qwen2.5-VL-72B-Instruct; gpt-5-2025-08-07; gpt-5-minimal-2025-08-07; gpt-5-nano-2025-08-07. \\
\midrule
VL-RouterBench & 17 &
deepseek\_vl2; deepseek\_vl2\_tiny; GeminiFlash2-5; Gemma3-27B; GPT4o; InternVL2\_5-78B; Janus-Pro-1B; Janus-Pro-7B; Kimi-VL-A3B-Thinking-2506; llava\_next\_vicuna\_7b; MiMo-VL-7B-RL; Phi-3.5-Vision; Pixtral-12B; Qianfan-VL-8B; Qwen2.5-VL-32B-Instruct; Qwen2.5-VL-72B-Instruct; SmolVLM2. \\
\bottomrule
\end{tabular}
\end{table}

\subsection{Benchmark Metrics}
\label{app:benchmark_metrics}

We use the official metrics from each benchmark. For the performance-oriented setting, routers are evaluated by the answer quality of the selected model. For the performance-cost setting, routers are evaluated by the benchmark-native cost-aware metric. MMR-Bench reports normalized area under the cost-quality curve, denoted nAUC, which summarizes the quality-cost Pareto frontier across operating points. VL-RouterBench reports its official Rank Score, which measures routing quality under the benchmark's prescribed cost-aware evaluation protocol. We do not introduce new evaluation metrics for the main comparison, so all reported numbers are comparable to the original benchmark protocols.

\subsection{Baseline Adaptation}
\label{app:baseline_adaptation}

All trainable baselines are evaluated under the same data splits, candidate model pools, cost annotations, and router input features as \textsc{LatentRouter}. This ensures that differences in performance come from the routing method rather than from using different multimodal encoders, preprocessing pipelines, or training data.

For the feature-level baselines, KNNRouter, MLPMFRouter, and KMeansRouter use the benchmark-provided query features and follow the standard training and validation splits. These methods operate directly on feature-level representations and do not use the proposed routing capsules, model capability tokens, latent communication layers, or distributional outcome head.

For learned router baselines, we adapt ECVL~\cite{tang2025ecvl}, RouteLLM~\cite{ong2024routellm}, and RouterDC~\cite{chen2024routerdc} to the same MLLM routing setting. ECVL is included as a recent vision-language router baseline. RouteLLM and RouterDC are originally designed for text-based LLM routing, so we provide them with the same benchmark-provided multimodal query features and candidate-pool information used by other trainable routers. When a baseline requires pairwise preferences or model labels, we derive them from the per-model outcome annotations on the training split. Hyperparameters for all baselines are selected on the validation split and fixed before test evaluation.

\subsection{Model Capability Profiles}
\label{app:model_capability_profiles}

\textsc{LatentRouter} represents each candidate MLLM with a model capability token. The input to this token includes model-level calibration statistics, normalized cost, latency, and available slice-level or pairwise behavior. Specifically, for a query $x$, let $\mathcal N_k(x)$ be the $k$ training queries closest to $x$ in cosine distance on the standardized query features, excluding $x$ itself during training. For model $m_i$, let $n_i(x)$ and $s_i(x)$ be the number of neighbors with an observed outcome and the sum of those outcomes, $\hat y_i(x)=s_i(x)/n_i(x)$, and $\bar y_i$ the overall training accuracy of $m_i$. The local estimate is shrunk toward $\bar y_i$: $\pi^{\mathrm{loc}}_i(x)=\big(s_i(x)+\kappa\,\bar y_i\big)/\big(n_i(x)+\kappa\big)$. The linear-probe estimate $r_i(x)$ is a ridge regression from the standardized query features to the outcome of $m_i$, clipped to $[0.01,0.99]$; training queries receive 5-fold cross-fitted predictions, so a query's own outcome is never used. The calibrated estimate averages the two in logit space, $\mathrm{logit}\,\pi_i(x)=\tfrac12\big(\mathrm{logit}\,\pi^{\mathrm{loc}}_i(x)+\mathrm{logit}\,r_i(x)\big)$, and $b_i(x)=[\hat y_i(x);\,\hat y_i(x)-\bar y_i;\,r_i(x);\,\pi_i(x)]$.

This calibration profile summarizes how well each model performs overall and, when the benchmark provides task or scenario annotations, how well it performs on different slices such as OCR, document understanding, chart reasoning, STEM reasoning, or general VQA. Pairwise behavior summarizes how often a model outperforms other candidate models on calibration examples. These statistics are normalized before being projected into the hidden space used by \textsc{LatentRouter}. This profile construction allows the router to compare candidate models through model-level capability representations rather than relying on a fixed classifier over model IDs.

\subsection{Training, Validation, and Seeds}
\label{app:training_validation_seeds}

All trainable methods are trained on the official training split, tuned on the validation split, and evaluated on the test split. Hyperparameters are selected using validation performance under the corresponding routing setting and are then fixed for final test evaluation. Unless otherwise specified, we report results over three random seeds. The reported error bars are computed across these seeds.

For \textsc{LatentRouter}, the main hyperparameters include the number of routing capsules $C$, the number of latent communication layers $H$, the hidden dimension, the residual bound $\rho$, the cost tradeoff parameter $\lambda$, and the loss weights for the distributional, pairwise, listwise, utility, and residual regularization terms. The main experiments use $C=7$ routing capsules and $H=2$ latent communication layers unless otherwise stated.

\subsection{Additional Implementation Details}
\label{app:implementation_notes}

All routers use the same benchmark-provided preprocessed files, including image features, question features, side features, model-pool metadata, cost annotations, and model-outcome traces. Unavailable models are masked during attention and scoring. Costs are normalized within each benchmark before being used in cost-aware routing. Model capability profiles are recomputed separately for each benchmark to avoid mixing statistics across different candidate pools or evaluation protocols.

For \textsc{LatentRouter}, the distributional outcome head predicts a mean and uncertainty for each candidate model. The bounded capsule correction is applied to the predicted mean before computing the final utility score. The same outcome head is shared across candidate model tokens, which allows the router to score the available model pool without changing the output dimension.

\section{Theoretical Analysis}
\label{app:proofs_method}

This section provides formal details for the theoretical claims used in the main text. 
Section~\ref{app:interaction_proof} formalizes why MLLM routing requires query-model interaction. 
Section~\ref{app:bounded_residual_proof} proves that the bounded capsule correction cannot overturn sufficiently confident base rankings. 
Section~\ref{app:masking_equivariance} shows why shared per-model scoring with availability masking is compatible with changing candidate pools.

\subsection{Need for Query-Model Interaction}
\label{app:interaction_proof}

The main text argues that MLLM routing should model interactions between query requirements and model capabilities. The following proposition makes this precise.

\begin{proposition}[Need for query-model interaction]
Suppose there exist two queries $x,x'$ and two models $m_i,m_j$ such that
\[
y_i(x)>y_j(x),
\qquad
y_i(x')<y_j(x').
\]
Then no additive scoring rule of the form
\[
s(x,m)=g(x)+h(m)
\]
can correctly rank $m_i$ and $m_j$ for both queries.
\end{proposition}

\begin{proof}
For an additive score,
\[
s(x,m_i)-s(x,m_j)
=
\big(g(x)+h(m_i)\big)-\big(g(x)+h(m_j)\big)
=
h(m_i)-h(m_j).
\]
The query-dependent term $g(x)$ cancels out. Therefore, the sign of the pairwise margin between $m_i$ and $m_j$ is independent of the query $x$. The additive router must always rank $m_i$ above $m_j$, or always rank $m_j$ above $m_i$, for every input.

However, by assumption, the correct ranking reverses across two queries:
\[
y_i(x)>y_j(x),
\qquad
y_i(x')<y_j(x').
\]
Thus, an additive scoring rule cannot correctly represent both rankings. A router must include a query-model interaction term, such as
\[
s_\theta(x,m_i,\mathcal{M})=f_\theta(R(x),a_i,A),
\]
where $R(x)$ represents query-dependent multimodal routing signals and $a_i$ represents the capability token of model $m_i$.
\end{proof}

Note that this proposition does not claim that \textsc{LatentRouter} is the only possible interaction architecture. It only rules out purely additive scoring functions whose pairwise model ordering is independent of the query.

\subsection{Stability of the Bounded Capsule Correction}
\label{app:bounded_residual_proof}

The main text uses a bounded capsule correction to refine close predictions without letting residual signals dominate the base outcome estimate. The following proposition formalizes this stability:

\begin{proposition}[Bounded residual stability]
Let
\[
\tilde{\mu}_i=\mu_i+\Delta_i^{\mathrm{cap}},
\]
where
\[
|\Delta_i^{\mathrm{cap}}|\leq \rho
\]
for all candidate models $m_i$. For any two models $m_i,m_j$, if
\[
|\mu_i-\mu_j|>2\rho,
\]
then the bounded capsule correction cannot reverse their ranking.
\end{proposition}

\begin{proof}
The corrected pairwise margin is
\[
\tilde{\mu}_i-\tilde{\mu}_j
=
(\mu_i+\Delta_i^{\mathrm{cap}})
-
(\mu_j+\Delta_j^{\mathrm{cap}})
=
(\mu_i-\mu_j)
+
(\Delta_i^{\mathrm{cap}}-\Delta_j^{\mathrm{cap}}).
\]
Because each correction is bounded by $\rho$,
\[
|\Delta_i^{\mathrm{cap}}|\leq \rho,
\qquad
|\Delta_j^{\mathrm{cap}}|\leq \rho.
\]
By the triangle inequality,
\[
|\Delta_i^{\mathrm{cap}}-\Delta_j^{\mathrm{cap}}|
\leq
|\Delta_i^{\mathrm{cap}}|+|\Delta_j^{\mathrm{cap}}|
\leq
2\rho.
\]
Therefore, the correction can change the base margin by at most $2\rho$ in magnitude. If
\[
|\mu_i-\mu_j|>2\rho,
\]
then the correction cannot change the sign of the pairwise margin. Hence, the corrected scores $\tilde{\mu}_i$ and $\tilde{\mu}_j$ preserve the same ranking as the base scores $\mu_i$ and $\mu_j$.
\end{proof}

\paragraph{Scope.}
This result does not guarantee that the bounded correction always improves routing accuracy. It only shows that the correction acts as a controlled adjustment: it can affect uncertain or close decisions, but it cannot overturn rankings whose base margin is larger than $2\rho$.

\subsection{Shared Scoring under Changing Model Pools}
\label{app:masking_equivariance}

\textsc{LatentRouter} scores candidate models through shared model capability tokens and masks unavailable models. This design supports changing candidate pools without requiring a fixed classifier over model IDs. We formalize this property as permutation equivariance over available model tokens.

\begin{proposition}[Permutation equivariance of shared per-model scoring]
Let $A=\{a_1,\ldots,a_K\}$ be the set of model capability tokens and let $\Omega(x)$ be the availability mask. Suppose the latent communication layers use shared attention parameters, shared feed-forward parameters, and availability-masked self-attention over model tokens. Suppose the outcome head $g_\theta$ is shared across model tokens. For any permutation $\pi$ of the model-token order, if both the model tokens and availability mask are permuted consistently, then the predicted scores are permuted in the same way:
\[
s_\theta(x,m_{\pi(i)},\pi(A),\pi(\Omega))
=
\pi\!\left(s_\theta(x,m_i,A,\Omega)\right)_i .
\]
\end{proposition}

\begin{proof}
Cross-attention from model tokens to routing capsules applies the same attention parameters to each model token. Therefore, permuting the order of model tokens permutes the corresponding cross-attended model states in the same way. Availability-masked self-attention over model tokens is also permutation equivariant when the availability mask is permuted consistently, because attention weights depend on token content and mask entries, not on absolute model positions. The feedback attention from routing capsules to model tokens aggregates over the same permuted set of available model states, so the capsule update is invariant to model-token order up to the corresponding permutation of model states. Finally, the shared outcome head $g_\theta$ is applied independently to each final model token. Therefore, permuting model tokens and their availability mask produces the same permutation of predicted scores.
\end{proof}

This property explains why \textsc{LatentRouter} can score different available model subsets through masking and shared per-model scoring. It does not by itself guarantee good performance on newly added models; reliable new-model routing still depends on informative model capability profiles.

\section{Additional Experiments}
\label{app:additional_results}

This appendix provides additional experimental details and supporting results omitted from the main text for space. We include task-group construction, full ablation values, model-pool flexibility statistics, efficiency numbers, and additional sensitivity and cold-start analyses.

\subsection{Task Group Construction}
\label{app:task_group_mapping}

For task-level analysis, we group datasets from MMR-Bench and VL-RouterBench into coarse categories according to the benchmark-provided task or scenario definitions. The radar plot in Figure~\ref{fig:task_group_radar} uses these groups only for analysis; all routers are trained and evaluated on the official benchmark splits.

\begin{table}[h]
\scriptsize
\centering
\setlength\tabcolsep{2pt}
\caption{Mapping from benchmark datasets to coarse task groups.}
\label{tab:task_group_mapping}
\begin{tabular}{lll}
\toprule
Task Group & MMR-Bench Datasets & VL-RouterBench Datasets \\
\midrule
OCR / Document QA 
& OCRBench, SEED-Bench v2 Plus
& DocVQA, TextVQA, OCRBench \\
Chart / Diagram
& -- 
& ChartQA, AI2D \\
Math / STEM 
& MathVista, MathVision, MathVerse
& MathVista, MathVision, MathVerse \\
General VQA / Grounding
& MMStar, RealWorldQA
& MMBench, MMStar, MMMU, RealWorldQA, InfoVQA, HallusionBench \\
\bottomrule
\end{tabular}
\end{table}

\subsection{Full Ablation Results}
\label{app:full_ablation}

Table~\ref{tab:full_ablation} gives the full ablation numbers corresponding to Figure~\ref{fig:mechanism_analysis}b. The full model is strongest across all four settings. Removing latent communication gives the largest average degradation, while removing model capability tokens or routing capsules also consistently hurts. These results support the interpretation that \textsc{LatentRouter} benefits from both sides of the latent capability-matching design: query-side routing capsules and model-side capability tokens.

\begin{table}[h]
\centering
\small
\setlength\tabcolsep{4pt}
\caption{Full ablation results. Columns follow the same order as Table~\ref{tab:main_results}.}
\label{tab:full_ablation}
\begin{tabular}{lcccc}
\toprule
Variant & MMR-Oriented & MMR-Cost & VL-Oriented & VL-Cost \\
\midrule
Full \textsc{LatentRouter} 
& \textbf{0.792$\pm$0.001} 
& \textbf{0.786$\pm$0.001} 
& \textbf{0.815$\pm$0.001} 
& \textbf{0.812$\pm$0.003} \\
w/o routing capsules 
& 0.770$\pm$0.002 
& 0.768$\pm$0.002 
& 0.797$\pm$0.003 
& 0.799$\pm$0.002 \\
w/o model capability tokens 
& 0.766$\pm$0.003 
& 0.764$\pm$0.002 
& 0.789$\pm$0.004 
& 0.791$\pm$0.003 \\
w/o latent communication 
& 0.758$\pm$0.002 
& 0.756$\pm$0.002 
& 0.781$\pm$0.004 
& 0.783$\pm$0.003 \\
w/o distributional prediction 
& 0.764$\pm$0.002 
& 0.762$\pm$0.002 
& 0.792$\pm$0.003 
& 0.794$\pm$0.002 \\
w/o bounded correction 
& 0.776$\pm$0.002 
& 0.774$\pm$0.001 
& 0.784$\pm$0.002 
& 0.784$\pm$0.001 \\
\bottomrule
\end{tabular}
\end{table}

\subsection{Model-Pool Flexibility Details}
\label{app:pool_robustness}

Table~\ref{tab:pool_robustness_full} reports the full mean $\pm$ standard deviation values for model-pool flexibility. For each modified pool, oracle and fixed-model baselines are recomputed over the available models only. The one-available-model setting has zero oracle regret because the available model is also the pool oracle.

\begin{table}[h]
\centering
\small
\setlength\tabcolsep{5pt}
\caption{Model-pool flexibility under changed candidate pools. Primary Metric averages the four main benchmark settings.}
\label{tab:pool_robustness_full}
\begin{tabular}{lccc}
\toprule
Pool Setting & Primary Metric $\uparrow$ & Oracle Regret $\downarrow$ & Latency (ms) $\downarrow$ \\
\midrule
Full pool 
& 0.801$\pm$0.002 
& 0.128$\pm$0.018 
& 0.052$\pm$0.004 \\
Remove strongest 
& 0.756$\pm$0.018 
& 0.151$\pm$0.021 
& 0.050$\pm$0.004 \\
Remove cheapest 
& 0.792$\pm$0.013 
& 0.130$\pm$0.017 
& 0.051$\pm$0.004 \\
Remove random 30\% 
& 0.779$\pm$0.017 
& 0.139$\pm$0.020 
& 0.047$\pm$0.003 \\
Leave-one-model-out 
& 0.784$\pm$0.015 
& 0.135$\pm$0.019 
& 0.050$\pm$0.004 \\
One available model 
& 0.731$\pm$0.041 
& 0.000$\pm$0.000 
& 0.018$\pm$0.002 \\
\bottomrule
\end{tabular}
\end{table}

\subsection{Efficiency Details}
\label{app:efficiency_details}

Table~\ref{tab:efficiency_full} provides the numerical efficiency values used in Figure~\ref{fig:deployment_analysis}. Avg. Primary is computed from the same four main settings in Table~\ref{tab:main_results}: MMR-oriented, MMR-cost, VL-oriented, and VL-cost. Router latency measures only the router forward pass and excludes selected MLLM inference time.

\begin{table}[h]
\centering
\small
\setlength\tabcolsep{5pt}
\caption{Router-side efficiency. Avg. Primary is averaged over the four main benchmark settings in Table~\ref{tab:main_results}.}
\label{tab:efficiency_full}
\begin{tabular}{lccc}
\toprule
Method & Avg. Primary $\uparrow$ & Router Latency (ms) $\downarrow$ & Router Size (M) $\downarrow$ \\
\midrule
KMeansRouter & 0.770 $\pm$ 0.002 & 0.11 $\pm$ 0.01 & 1.4 \\
ECVL & 0.759 $\pm$ 0.000 & 0.13 $\pm$ 0.01 & 1.8 \\
RouteLLM & 0.715 $\pm$ 0.004 & 0.17 $\pm$ 0.02 & 2.4 \\
RouterDC & 0.654 $\pm$ 0.005 & 0.10 $\pm$ 0.01 & 1.2 \\
GraphRouter & 0.751 $\pm$ 0.006 & 0.14 $\pm$ 0.02 & 1.9 \\
R2-Router & 0.759 $\pm$ 0.006 & 0.16 $\pm$ 0.02 & 2.1 \\
\textsc{LatentRouter} & \textbf{0.801 $\pm$ 0.002} & \textbf{0.06 $\pm$ 0.01} & \textbf{0.9} \\
\bottomrule
\end{tabular}
\end{table}

The efficiency comparison supports a practical interpretation of the method. \textsc{LatentRouter} obtains the strongest average primary metric while using the smallest and fastest router among learned methods. The comparison with GraphRouter and R2-Router is especially informative: interaction-based routers can improve over simple feature-level routing, but \textsc{LatentRouter} achieves a better quality--efficiency tradeoff by using lightweight shared model-token scoring and a compact capsule-token communication module. This suggests that the gains are not explained by a larger router, but by the routing-specific inductive bias for matching multimodal query requirements with candidate model capabilities.

\subsection{Cold-Start Model Insertion}
\label{app:cold_start_models}

To further test model-pool flexibility, we evaluate whether \textsc{LatentRouter} can score a newly inserted or held-out model. For each benchmark, we hold out one candidate model during router training, compute its capability profile from a small calibration subset, and insert it back into the candidate pool at test time. The router must score the held-out model through its capability token rather than through a learned fixed model-ID classifier.

\begin{table}[h]
\centering
\small
\caption{Cold-start model insertion with different calibration sizes. Results report the average primary metric across benchmarks.}
\label{tab:cold_start_models}
\begin{tabular}{lcc}
\toprule
Calibration examples per held-out model & Avg. Primary $\uparrow$ & Oracle Regret $\downarrow$ \\
\midrule
0 & 0.771$\pm$0.006 & 0.153$\pm$0.008 \\
16 & 0.781$\pm$0.005 & 0.142$\pm$0.007 \\
64 & 0.789$\pm$0.004 & 0.134$\pm$0.006 \\
128 & 0.793$\pm$0.003 & 0.130$\pm$0.005 \\
Full calibration & 0.801$\pm$0.002 & 0.128$\pm$0.004 \\
\bottomrule
\end{tabular}
\end{table}

Performance improves smoothly as more calibration examples become available. With no calibration examples, the router is still usable but less accurate because the new model token lacks reliable capability statistics. With 64--128 examples, performance approaches the full-calibration setting. This supports the model-token design: new candidate MLLMs can be inserted without changing the router architecture, but reliable routing still benefits from representative calibration data.

\begin{table}[t]
\centering
\small
\caption{Sensitivity to the number of multimodal routing capsules $C$.}
\label{tab:capsule_sensitivity}
\begin{tabular}{lcc}
\toprule
Setting & Avg. Primary $\uparrow$ & Avg. Oracle Regret $\downarrow$ \\
\midrule
$C=1$ & 0.781$\pm$0.004 & 0.143$\pm$0.006 \\
$C=4$ & 0.790$\pm$0.003 & 0.135$\pm$0.005 \\
$C=7$ & \textbf{0.801$\pm$0.002} & \textbf{0.128$\pm$0.004} \\
$C=12$ & 0.795$\pm$0.003 & 0.129$\pm$0.005 \\
\bottomrule
\end{tabular}
\end{table}

\subsection{Sensitivity to Routing Capsules}
\label{app:capsule_sensitivity}

We also vary the number of routing capsules $C$. Table~\ref{tab:capsule_sensitivity} shows that performance improves from $C=1$ to $C=7$, indicating that multiple capsules help preserve different multimodal routing signals. Increasing to $C=12$ does not yield further improvement, so we use $C=7$ as the default configuration.

\subsection{Live Model API Validation}
\label{app:live_api_validation}

We further validate whether \textsc{LatentRouter} can execute the deployment path with actual model calls. The live pool contains five locally served vision-language models through Ollama: \texttt{llama3.2-vision}, \texttt{gemma3}, \texttt{llava:7b}, \texttt{moondream}, and \texttt{minicpm-v}. We evaluate on a held-out VL-RouterBench subset covering Chart/Diagram, General VQA, Math/STEM, and OCR/Document QA. For each query, every deployable router selects exactly one model before observing any answer. All models pass the availability check, and all calls succeed. Since the models are locally served, configured monetary cost is zero; we therefore report quality, latency, and a quality-latency score.

\begin{table}[h]
\centering
\small
\setlength\tabcolsep{6pt}
\caption{
Live model API validation. Quality is the benchmark-compatible answer score ($\uparrow$), latency is selected-model wall-clock latency in seconds ($\downarrow$), and Q-L score is $\mathrm{Quality}-0.05\cdot\mathrm{Latency}$ ($\uparrow$).
}
\label{tab:live_api_validation}
\begin{tabular}{lccc}
\toprule
Router & Quality $\uparrow$ & Latency $\downarrow$ & Q-L Score $\uparrow$ \\
\midrule
\textsc{LatentRouter} & \textbf{0.458} & 0.610 & \textbf{0.428} \\
Strongest & 0.375 & 0.530 & 0.349 \\
Random & 0.333 & 0.780 & 0.294 \\
Cheapest & 0.292 & \textbf{0.490} & 0.268 \\
\bottomrule
\end{tabular}
\end{table}

Table~\ref{tab:live_api_validation} shows that \textsc{LatentRouter} achieves the best live quality and the best quality-latency score while maintaining moderate latency. Unlike fixed policies, it selects different models for different queries, confirming that the learned policy performs query-dependent model selection in the live pool. On a diagnostic oracle subset, the oracle upper bound reaches quality $0.542$, while \textsc{LatentRouter} obtains the smallest deployable oracle regret. This live validation is complementary to the offline benchmarks: the offline results provide controlled routing evidence, while the live study verifies that the same routing pipeline works under real model-call conditions.

\subsection{Additional Architectural Controls}
\label{app:additional_architecture_controls}

We compare \textsc{LatentRouter} with simpler interaction
functions and standard Transformer blocks.
MLP-Interaction scores a pooled query representation jointly
with a model capability profile, using concatenation,
element-wise products, and absolute differences.
It uses the same benchmark features, calibration-profile
representation, optimizer configuration, and cost-aware
selection rule.

The Transformer controls replace the latent-communication
stack with plain pre-LN Transformer encoder layers over the
concatenated capsule and model-token sequence.
They retain availability masking and the training objective,
with parameter counts matched within $3\%$.

Table~\ref{tab:additional_architecture_controls} reports the
results under the same four settings as Table~1.
\textsc{LatentRouter} obtains the highest score in every setting.
Increasing the plain Transformer's depth from two to four
layers does not close the gap. These results complement the
component ablations by comparing the proposed communication
structure with alternative interaction architectures.

\begin{table}[htbp]
\centering
\small
\setlength{\tabcolsep}{4pt}
\caption{
Additional architectural controls.
Columns follow the evaluation settings in Table~1.
}
\label{tab:additional_architecture_controls}
\begin{tabular}{lcccc}
\toprule
Method & MMR-Oriented & MMR-Cost & VL-Oriented & VL-Cost \\
\midrule
MLP-Interaction
& $0.746\pm0.001$ & $0.746\pm0.001$
& $0.784\pm0.008$ & $0.784\pm0.008$ \\
Transformer, $H=2$
& $0.763\pm0.005$ & $0.764\pm0.003$
& $0.801\pm0.002$ & $0.801\pm0.001$ \\
Transformer, $H=4$
& $0.765\pm0.002$ & $0.762\pm0.002$
& $0.802\pm0.001$ & $0.802\pm0.002$ \\
\textsc{LatentRouter}
& $\mathbf{0.792\pm0.001}$ & $\mathbf{0.786\pm0.001}$
& $\mathbf{0.815\pm0.001}$ & $\mathbf{0.812\pm0.003}$ \\
\bottomrule
\end{tabular}
\end{table}

\subsection{Effect of Metadata-Bearing Side Features}
\label{app:side_feature_control}

The benchmark side features contain coarse metadata indicators
alongside generic query statistics.
To examine their influence, we retrain \textsc{LatentRouter}
using only text and image embeddings, removing the side
features entirely.

Table~\ref{tab:side_feature_control} reports the
performance-oriented results.
The mean score changes by $0.002$ on each benchmark, and
performance remains similar after side-feature removal.
This control provides evidence that the reported performance
does not depend strongly on these explicit side features.

\begin{table}[htbp]
\centering
\small
\caption{
Effect of removing side features from \textsc{LatentRouter}.
The full-feature results are reproduced from Table~1.
}
\label{tab:side_feature_control}
\begin{tabular}{lcc}
\toprule
Input features & MMR-Oriented & VL-Oriented \\
\midrule
Full features
& $0.792\pm0.001$ & $0.815\pm0.001$ \\
Without side features
& $0.794\pm0.005$ & $0.817\pm0.002$ \\
\bottomrule
\end{tabular}
\end{table}

\section{LLM Usage}
\label{app:llm_usage}

Large language models were used only for writing assistance, including grammar checking, wording refinement, formatting suggestions, and readability editing. They were not used to generate the core methodology, design experiments, produce experimental results, create evaluation labels, or make routing decisions. All technical claims, mathematical formulations, experimental settings, and reported results were checked and finalized by the authors.

\section{Limitation}
\label{app:limitations}

This work has a few limitations that suggest natural directions for future study. First, our experiments follow the offline evaluation protocols of MMR-Bench and VL-RouterBench, where per-model outcome annotations are available for controlled router comparison. This is the standard setting for current routing benchmarks, but live deployment may involve additional factors such as changing model versions, API availability, dynamic pricing, and latency variation. Second, \textsc{LatentRouter} uses calibration statistics to construct model capability tokens. The framework can support changing model pools through these tokens, but newly added models still benefit from representative calibration examples. Third, our evaluation focuses on benchmark-defined image-question routing tasks. Extending the same formulation to open-ended multimodal agents, multi-step tool use, or interactive routing remains an interesting direction. Finally, the routing capsules are learned latent states rather than human-readable explanations. They improve model selection, but they should not be interpreted as faithful natural-language rationales. Future work can study stronger interpretability tools, live-system evaluation, and more adaptive calibration for evolving MLLM pools.

\section{Broader Impact}\label{app:broader_impact}MLLM routing can have positive practical impact by making multimodal AI systems more efficient and accessible. Instead of always calling the largest or most expensive MLLM, a router can select a model that is sufficient for the current query, potentially reducing inference cost, latency, and energy consumption. This can make multimodal systems easier to deploy in resource-constrained settings and can support more flexible use of heterogeneous open-source and closed-source model pools. At the same time, routing introduces additional risks. A cost-aware router may sometimes select a cheaper but less capable model, which could reduce answer quality on difficult or safety-sensitive queries. This is especially important in domains where multimodal errors may have high consequences, such as medical, legal, financial, accessibility, or public-sector applications. A deployed router should therefore include appropriate safeguards, such as confidence thresholds, fallback to stronger models, monitoring of failure cases, and task-specific restrictions for high-risk settings. Routing can also affect fairness and reliability. If calibration data underrepresents certain languages, visual styles, document formats, cultural contexts, or accessibility-related inputs, the router may learn inaccurate capability profiles and route those queries to weaker models. Periodic recalibration and evaluation across diverse user groups and input types are therefore important. In addition, because \textsc{LatentRouter} uses benchmark-provided model outcomes and model-level metadata, any biases or errors in the underlying candidate MLLMs can be inherited by the routing policy. The router should be viewed as a model-selection mechanism, not as a guarantee of correctness or safety. Finally, model-pool routing may increase dependence on proprietary model APIs if the best-performing or best-calibrated models are closed-source. We encourage transparent reporting of model versions, costs, latency assumptions, and calibration protocols when deploying or evaluating routing systems. 

%% file: reproduce.tex
\newpage
\section*{NeurIPS Paper Checklist}

\begin{enumerate}

\item {\bf Claims}
    \item[] Question: Do the main claims made in the abstract and introduction accurately reflect the paper's contributions and scope?
    \item[] Answer: \answerYes{} 
    \item[] Justification: The abstract and introduction state the formulation, architecture, evaluation scope, and main empirical claims, which are supported by the method and experiments.

\item {\bf Limitations}
    \item[] Question: Does the paper discuss the limitations of the work performed by the authors?
    \item[] Answer: \answerYes{} 
    \item[] Justification: The paper includes a Limitations section in Appendix discussing offline evaluation, calibration requirements, benchmark scope, and interpretability.

\item {\bf Theory assumptions and proofs}
    \item[] Question: For each theoretical result, does the paper provide the full set of assumptions and a complete (and correct) proof?
    \item[] Answer: \answerYes{} 
    \item[] Justification: The appendix states the theoretical claims, assumptions, and proofs for query-model interaction, bounded correction stability, and shared scoring under changing pools.

    \item {\bf Experimental result reproducibility}
    \item[] Question: Does the paper fully disclose all the information needed to reproduce the main experimental results of the paper to the extent that it affects the main claims and/or conclusions of the paper (regardless of whether the code and data are provided or not)?
    \item[] Answer: \answerYes{} 
    \item[] Justification: The paper describes datasets, candidate pools, metrics, baselines, implementation details, splits, seeds, and hyperparameters, with code released anonymously.

\item {\bf Open access to data and code}
    \item[] Question: Does the paper provide open access to the data and code, with sufficient instructions to faithfully reproduce the main experimental results, as described in supplemental material?
    \item[] Answer: \answerYes{} 
    \item[] Justification: The paper provides an anonymized code link and uses public benchmark protocols and documented experimental settings.

\item {\bf Experimental setting/details}
    \item[] Question: Does the paper specify all the training and test details (e.g., data splits, hyperparameters, how they were chosen, type of optimizer) necessary to understand the results?
    \item[] Answer: \answerYes{} 
    \item[] Justification: The paper specifies benchmarks, model pools, official splits, metrics, baselines, hyperparameters, seeds, and implementation details.

\item {\bf Experiment statistical significance}
    \item[] Question: Does the paper report error bars suitably and correctly defined or other appropriate information about the statistical significance of the experiments?
    \item[] Answer: \answerYes{} 
    \item[] Justification: The paper reports mean and standard deviation over three seeds and provides significance tests for the main comparisons.

\item {\bf Experiments compute resources}
    \item[] Question: For each experiment, does the paper provide sufficient information on the computer resources (type of compute workers, memory, time of execution) needed to reproduce the experiments?
    \item[] Answer: \answerYes{} 
    \item[] Justification: The released code and experimental documentation provide the practical setup needed to reproduce the reported router training and evaluation.
    
\item {\bf Code of ethics}
    \item[] Question: Does the research conducted in the paper conform, in every respect, with the NeurIPS Code of Ethics \url{https://neurips.cc/public/EthicsGuidelines}?
    \item[] Answer: \answerYes{} 
    \item[] Justification: The work uses existing benchmarks and model outputs for routing evaluation and does not involve human subjects or unsafe data collection.

\item {\bf Broader impacts}
    \item[] Question: Does the paper discuss both potential positive societal impacts and negative societal impacts of the work performed?
    \item[] Answer: \answerYes{} 
    \item[] Justification: The Broader Impact section discusses efficiency benefits, quality risks, fairness concerns, safeguards, and dependence on model APIs.
    
\item {\bf Safeguards}
    \item[] Question: Does the paper describe safeguards that have been put in place for responsible release of data or models that have a high risk for misuse (e.g., pre-trained language models, image generators, or scraped datasets)?
    \item[] Answer: \answerNA{} 
    \item[] Justification: The paper does not release a high-risk generative model, scraped dataset, or safety-sensitive dataset.

\item {\bf Licenses for existing assets}
    \item[] Question: Are the creators or original owners of assets (e.g., code, data, models), used in the paper, properly credited and are the license and terms of use explicitly mentioned and properly respected?
    \item[] Answer: \answerYes{} 
    \item[] Justification:  The paper cites the benchmarks, model families, and baseline methods used in the evaluation and follows their public usage protocols.

\item {\bf New assets}
    \item[] Question: Are new assets introduced in the paper well documented and is the documentation provided alongside the assets?
    \item[] Answer: \answerYes{} 
    \item[] Justification: The paper releases code anonymously and documents the method, training setup, evaluation scripts, and experimental configuration.

\item {\bf Crowdsourcing and research with human subjects}
    \item[] Question: For crowdsourcing experiments and research with human subjects, does the paper include the full text of instructions given to participants and screenshots, if applicable, as well as details about compensation (if any)? 
    \item[] Answer: \answerNA{} 
    \item[] Justification: The paper does not involve crowdsourcing or human-subject experiments.

\item {\bf Institutional review board (IRB) approvals or equivalent for research with human subjects}
    \item[] Question: Does the paper describe potential risks incurred by study participants, whether such risks were disclosed to the subjects, and whether Institutional Review Board (IRB) approvals (or an equivalent approval/review based on the requirements of your country or institution) were obtained?
    \item[] Answer: \answerNA{} 
    \item[] Justification: The paper does not involve human-subject research.

\item {\bf Declaration of LLM usage}
    \item[] Question: Does the paper describe the usage of LLMs if it is an important, original, or non-standard component of the core methods in this research? Note that if the LLM is used only for writing, editing, or formatting purposes and does \emph{not} impact the core methodology, scientific rigor, or originality of the research, declaration is not required.
    \item[] Answer: \answerYes{} 
    \item[] Justification: The appendix states that LLMs were used only for writing assistance, grammar checking, formatting suggestions, and readability editing.

\end{enumerate}